\documentclass{ecai} 

\usepackage{latexsym}
\usepackage{amssymb}
\usepackage{amsmath}
\usepackage{amsthm}
\usepackage{booktabs}
\usepackage{enumitem}
\usepackage{graphicx}
\usepackage{color}
\usepackage[normalem]{ulem}
\usepackage{algorithm2e}
\usepackage{subfig}
\usepackage{multirow}

\newcommand{\argmax}{\arg\,\max}

\newtheorem{theorem}{Theorem}

\newcommand{\BibTeX}{B\kern-.05em{\sc i\kern-.025em b}\kern-.08em\TeX}

\begin{document}


\begin{frontmatter}


\paperid{1867} 


\title{Scalable Variational Causal Discovery\\Unconstrained by Acyclicity}


\author{\fnms{Nu}~\snm{Hoang}\thanks{Corresponding Author. Email: nu.hoang@deakin.edu.au}}
\author{\fnms{Bao}~\snm{Duong}}
\author{\fnms{Thin}~\snm{Nguyen}} 

\address{Applied Artificial Intelligence Institute (A\textsuperscript{2}I\textsuperscript{2}), Deakin University, Australia.}


\begin{abstract}
Bayesian causal discovery offers the power to quantify epistemic uncertainties among a broad range of structurally diverse causal theories potentially explaining the data, represented in forms of directed acyclic graphs (DAGs). However, existing methods struggle with efficient DAG sampling due to the complex acyclicity constraint. In this study, we propose a scalable Bayesian approach to effectively learn the posterior distribution over causal graphs given observational data thanks to the ability to generate DAGs without explicitly enforcing acyclicity. Specifically, we introduce a novel differentiable DAG sampling method that can generate a valid acyclic causal graph by mapping an unconstrained distribution of implicit topological orders to a distribution over DAGs. Given this efficient DAG sampling scheme, we are able to model the posterior distribution over causal graphs using a simple variational distribution over a continuous domain, which can be learned via the variational inference framework. Extensive empirical experiments on both simulated and real datasets demonstrate the superior performance of the proposed model compared to several state-of-the-art baselines.
\end{abstract}

\end{frontmatter}


\section{Introduction}

Causal inference \citep{judea2009causality} offers a powerful tool for tackling critical research questions in diverse fields, such as policy decision-making, experimental design, and enhancing AI trustworthiness. However, current causal inference algorithms typically require an input of a directed acyclic graph (DAG), encapsulating causal relationships among variables of interest. Unfortunately, identifying the true causal DAG often necessitates extensive experimentation, which can be time-consuming and ethically problematic in certain situations, hindering the application of causal inference to high-dimensional problems. Therefore, there is a pressing need to explore methods for discovering the causal DAG solely from observational data, which is typically more readily available \citep{spirtes2001causation,jonas2017elements}. Nevertheless, a major hurdle in causal discovery using observational data is the non-identifiability issue of causal models when multiple
DAGs may induce the same observed data mainly due to scarce data,
model mis-specification, and limited capability of optimizers. Bayesian inference is a promising approach to mitigate this problem by estimating the posterior distribution over causal DAGs, allowing for capturing epistemic uncertainties in causal structure learning. Moreover, this richer representation can then be leveraged for various tasks, including active causal discovery, where we strategically collect additional data to refine our understanding of causal relationships \citep{agrawal2019abcdstrategy,tigas2023differentiable}. 

\begin{figure}[h]
\begin{centering}
\includegraphics[width=1.0\columnwidth]{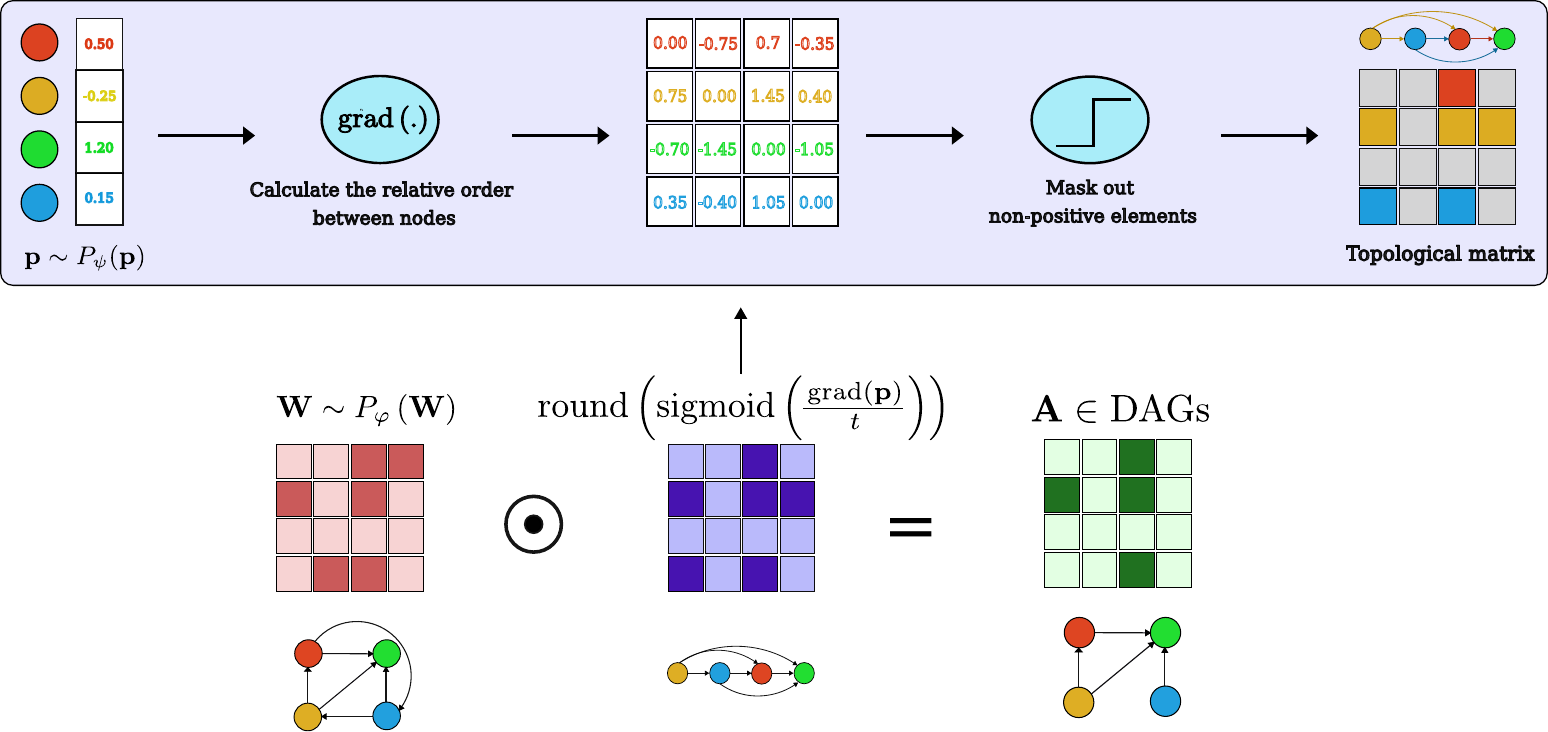}
\end{centering}
\caption{Overview of the proposed differentiable DAG sampling. Given two arbitrary probabilistic models over real vectors of $d$ nodes and binary matrices of $d\times d$, we first sample a priority scores vector $\mathbf{p}$ and a binary matrix $\mathbf{W}$. Then, we construct a binary adjacency matrix of a complete topological graph corresponding to $\mathbf{p}$, called a topological matrix, using the gradient operator followed by a tempered sigmoid function. The final DAG adjacency matrix $\mathbf{A}$ is the element-wise multiplication of $\mathbf{W}$ and the topological matrix derived from $\mathbf{p}$.\newline}
\label{fig:overview}
\end{figure}

Due to the exponential explosion in the number of possible DAGs with increasing variables \citep{chickering2004largesample}, inferring the posterior distribution becomes computationally impossible for large-scale problems. An efficient DAG sampler is therefore crucial to unlock the scalability of Bayesian causal discovery.
In particular,
several studies \citep{viinikka2020towards,deleu2022bayesian,atanackovic2023dyngfntowards}
have leveraged sequential models, e.g., Markov Chain Monte Carlo (MCMC)
or GFlowNets, to sample DAGs in combinatorial spaces, leading to high
computational demands. Recent advancements aim to enhance the inference
efficiency through the development of gradient optimization methods
for Bayesian structure learning \citep{cundy2021bcdnets2,lorch2021dibsdifferentiable,charpentier2022differentiable,annadani2023bayesdag},
which handle the acyclicity constraint by either integrating a smooth
DAG regularization into the objective function \citep{lorch2021dibsdifferentiable}
or relying on the permutation matrix-based DAG decomposition \citep{cundy2021bcdnets2,charpentier2022differentiable,annadani2023bayesdag}.
On the one hand, the inclusion of acyclicity regularization in the
objective function may introduce additional computational costs, impeding
scalability for high dimensional problems. For instance, the computational
complexity for computing NOTEARS \citep{zheng2018dagswith}, a widely
used DAG constraint function \citep{yu2019daggnndag,zheng2020learning,lachapelle2020gradientbased,ng2022maskedgradientbased},
grows cubically with the number of nodes. In addition, DAG regularizer
does not guarantee a complete DAG generation, potentially requiring
additional post-processing steps \citep{buhlmann2014camcausal}. On
the other hand, the permutation matrix-based approach, factorizing
the adjacency matrix into upper triangular and permutation matrices,
may expose high complexities due to the difficulty of approximating
the discrete distribution over the permutation matrix. For example,
various Bayesian studies \citep{cundy2021bcdnets2,charpentier2022differentiable,annadani2023bayesdag}
exploit the Gumbel-Sinkhorn operator, yet the process of turning a
parameter matrix into a permutation matrix incurs significant computational
overhead. 

To address computational limitations of current approaches, we introduce a novel differentiable DAG sampling method, which maps a constraint-free distribution of implicit topological
orders to a distribution over DAGs, eliminating the need for enforcing acyclicity explicitly. Inspired by \citep{yu2021dagswith}, we sample a DAG via generating a binary adjacency matrix representing any directed graph and a vector of nodes' priority scores. These scores implicitly define an topological order when sorted, inducing a topological matrix that effortlessly ensures acyclicity. Specifically, the topological matrix can be computed promptly through the pairwise differences of a vector of nodes' priority scores with a tempered sigmoid function. Finally, the proposed approach obtains a DAG through the element-wise multiplication of the generated binary adjacency matrix and the topological matrix as shown in  Figure~\ref{fig:overview}, reducing the time complexity for DAG sampling to quadratic with respect to the number of nodes, enabling DAG generation with thousands of nodes in a matter of milliseconds. Based on this competent DAG sampling scheme, we are able to model
the posterior distribution over DAGs using a simple variational distribution over a continuous domain, which can be learned via the variational inference framework. Our
main contributions are outlined as follows: 
\begin{itemize}
\item We present a novel differentiable probabilistic DAG model exhibiting scalable sampling for thousands of variables (Section~\ref{subsec:Differentiable-DAG-sampling}).
Furthermore, we provide theoretical justification substantiating the
correctness of the proposed method.
\item We introduce a fast and accurate Bayesian causal discovery method
built upon on our efficient DAG sampling and the variational inference
framework (Section~\ref{subsec:Variational-Inference-DAG}). The
proposed approach ensures the generation of valid acyclic causal structures
at any time during training, accompanied by a competitive runtime
compared to alternative methods.
\item We demonstrate the efficiency of the proposed method through extensive numerical experiments on both synthetic and real
datasets (Section~\ref{sec:Experiment})\footnote{Source code is available at \url{https://github.com/htn274/VCUDA}.}. The empirical results underscore the scalability of our DAG sampling
model and showcase its superior performance in Bayesian causal discovery
when incorporated with the variational inference framework compared
with several baselines. 
\end{itemize}

\section{Related Work}

\textbf{Discrete Optimization. }These methods encompass constraint-based
methods, score-based methods, and hybrid methods, which typically
search for the true causal graph within the original combinatorial
space of DAGs. Constraint-based methods \citep{spirtes2001causation,zhang2011kernelbased,colombo2014orderindependent}
depend on results from various conditional independence tests, while
score-based methods \citep{haughton1988onthe,heckerman1995learning,huang2018generalized}
optimize a predefined score to identify the final DAG by adding, removing,
or reversing edges. In the meantime, hybrid methods \citep{tsamardinos2006themaxmin,ogarrio2016ahybrid}
integrate both constraint-based and score-based techniques to trim
the search space, thus accelerating the overall learning process.

\textbf{Continuous Optimization. }Optimizing in the discrete space
of DAGs is known to be challenging due to the super exponential increase
in complexity with the number of variables. To address this issue,
several studies map the discrete space to a continuous one, thereby
unlocking the application of various continuous optimization techniques.
A pioneering study is NO TEARS \citep{zheng2018dagswith}, which introduced
a smooth function to evaluate the DAG-ness of a weighted adjacency
matrix. The causal structure learning problem is then tackled using
an augmented Lagrangian optimization method. Subsequent studies, inspired
by NO TEARS, have enhanced its efficiency by introducing low-complexity
DAG constraints \citep{lee2019scaling,zhang2022truncated}, or extending
it for non-linear functional models \citep{yu2019daggnndag,zheng2020learning,lachapelle2020gradientbased,ng2022maskedgradientbased}.
In contrast to NO TEARS, recent studies \citep{yu2021dagswith,massidda2023constraintfree}
introduce various DAG mapping functions, which facilitate direct optimization
within the DAG space. Therefore, these mapping functions provide more
scalable and direct approaches without the need to evaluate the DAG
constraint. 

\textbf{Bayesian Causal Structure Learning. }The above studies usually
output the Markov equivalence class (MEC) of the true DAG or a single
DAG, which may not adequately represent the uncertainty in certain
practical scenarios. To address this challenge, Bayesian causal discovery
methods produce a posterior distribution over causal DAGs. Several
studies demonstrate DAG sampling in the discrete space using either
Markov Chain Monte Carlo (MCMC) \citep{su2016improving,viinikka2020towards,kuipers2022efficient}
or GFlowNets \citep{deleu2022bayesian,atanackovic2023dyngfntowards}.
However, these approaches expose slow mixing and convergence. Recent
advancements aim for more efficient inference through the development
of gradient optimization methods for Bayesian structure learning.
However, existing studies still struggle in representing the acyclicity
constraint. For instance, DiBS \citep{lorch2021dibsdifferentiable}
exploits NO TEARS as a DAG regularizer and utilizes Stein variational
approach to learn the joint distribution over DAGs and causal model
parameters. However,  its scalability is limited for large graphs
due to the computational complexity associated with NO TEARS. In contrast,
BCDnets \citep{cundy2021bcdnets2} and DDS \citep{charpentier2022differentiable}
exploit the ordering-based DAG decomposition to parameterize the DAGs
distribution through the multiplication of upper triangular and permutation
matrices. For approximating a discrete distribution over the permutation
matrix, they utilized Gumbel-Sinkhorn operator, which poses a high
time complexity, i.e., cubic with respect to the number of node. To
reduce the complexity of the Gumbel-Sinkhorn operator, BayesDAG \citep{annadani2023bayesdag}
exploits No-Curl constraint \citep{yu2021dagswith} which can decrease
the number of iterations required for the Gumbel-Sinkhorn operator,
yet the scalability of this approach remains a challenge, when dealing
with large graphs.

\section{Preliminary}

\subsection{Problem formulation}

Let $\mathbf{X}\in\mathbb{R}^{n\times d}$ be an observational dataset
consisting of $n$ i.i.d. samples of $d$ random variables from a
joint distribution $P(X)$. The marginal joint distribution $P(X)$
factorizes according to a DAG $\mathcal{G}=\left\langle \mathbf{V},\mathbf{E}\right\rangle $,
where $\mathbf{V}=\{1,2,...,d\}$ is a set of nodes corresponding
to $d$ random variables and $\mathbf{E}$ is a set of edges representing
the dependency between nodes. In other words, $P(X)=\prod_{i=1}^{d}P(X_{i}|X_{\mathrm{pa}(i)})$,
where $X_{\mathrm{pa}(i)}$ denotes a set of parents of nodes $X_{i}$.
We can model the data of a node $X_{i}$ with a structural equation
model as follows:

\begin{equation}
X_{i}=g_{i}\left(f_{i}\left(X_{\mathrm{pa}(i)}\right),\epsilon_{i}\right)\label{eq:sem1},
\end{equation}
where $g_{i}$ and $f_{i}$ are deterministic functions and $\epsilon_{i}$
is an arbitrary noise. In this work, we consider a Gaussian additive
noise model (ANM) \citep{shimizu2006alinear,hoyer2008nonlinear} as
follows: 
\begin{equation}
X_{i}=f_{i}\left(X_{\mathrm{pa}(i)}\right)+\epsilon_{i},\epsilon_{i}\sim\mathcal{N}\left(0,\sigma^{2}\right),\label{eq:anm}
\end{equation}
The DAG $\mathcal{G}$ can be represented by a binary adjacency matrix
${\mathbf{A}\in\{0,1\}^{d\times d}}$, where $A_{ij}=1$ indicates an
edge from $X_{i}$ to $X_{j}$. As a result, Eq.~(\ref{eq:sem1})
can be rewritten as: 
\begin{equation}
X_{i}=f_{i}\left(A_{i}\circ \mathbf{X}\right)+\epsilon_{i},\label{eq:sem2}
\end{equation}
where $A_{i}$ is the $i^{\text{th}}$ column in the adjacency matrix
$\mathbf{A}$, $A_i \circ X$ is the $\mathbf{X}$ matrix with the columns corresponding to the 0-entries of $A_i$ being masked out. Intuitively, the matrix $\mathbf{A}$ plays as a mask
to extract parental nodes of $X_{i}$. Using a binary adjacency matrix
$\mathbf{A}$ to represent a DAG offers a more flexible approach to
model the functional relationship between each node $X_{i}$ and its
parents. 

Given $\mathbf{X}$, we aim to learn the adjacency matrix $\mathbf{A}$
of the DAG $\mathcal{G}$. Let $\mathbb{D}$ be the DAG space of $d$
nodes. Score-based methods usually solve an optimization problem by
maximizing a fitness score of a candidate graph and the data, i.e.,
$F(\mathbf{A},\mathbf{X})$: 
\begin{equation}
\max_{\mathbf{A}}F(\mathbf{A},\mathbf{X})\text{ s.t. }\mathbf{A}\in\mathbb{D}.\label{eq:score}
\end{equation}

\subsection{Bayesian causal structure learning}

Solving Eq.~(\ref{eq:score}) yields a single point DAG solution
that comes with practical limitations, particularly in addressing
the non-identifiability problem in DAG learning. This limitation stems
from the fact that the true causal DAG is only identifiable under
specific conditions. For instance, identifiability in the linear Gaussian
SEM holds in the equal variance noise setting \citep{peters2014identifiability}.
In real-world scenarios, the non-identifiability issue may surface
due to the limited number of observations, leading the point estimation
approach to converge toward an incorrect solution. Considering these
challenges, it proves beneficial to model the uncertainty in DAG learning,
where Bayesian learning emerges as a standard approach.

The ultimate goal of Bayesian causal structure learning methods is
to approximate the posterior distribution over the causal graph given
the observational data, denoted as $\mbox{\ensuremath{P(\mathcal{G}\mid\mathbf{X})}}$.
Using Bayes' rule, the posterior $\mbox{\ensuremath{P(\mathcal{G}\mid\mathbf{X})}}$
can be expressed through the prior distribution $P(\mathcal{G})$
and the marginal likelihood $\mbox{\ensuremath{P(\mathbf{X}\mid\mathcal{G})}}$
as follows: 
\begin{equation}
P(\mathcal{G}\mid\mathbf{X})=\frac{P(\mathbf{X}\mid\mathcal{G})P(\mathcal{G})}{\sum_{\mathcal{G}}P(\mathbf{X}\mid\mathcal{G})P(\mathcal{G})},\label{eq:bcd}
\end{equation}
where the marginal likelihood $P(\mathbf{X}\mid\mathcal{G})$ is defined
as a marginalization of the likelihood function over all possible
parameters for $\mathcal{G}$: 
\begin{equation}
P(\mathbf{X}\mid\mathcal{G})=\int P(\mathbf{X}\mid\mathcal{G},\boldsymbol{\theta}_{\mathcal{G}})P(\boldsymbol{\theta}_{\mathcal{G}}\mid\mathcal{G})d\mathbf{\boldsymbol{\theta}}_{\mathcal{G}}.\label{eq:likelihood}
\end{equation}

The main challenge in Bayesian causal discovery is the intractability
of the denominator in Eq.~(\ref{eq:bcd}) due to the expansive space
of DAGs.

\subsection{DAG representation}

To deal with the computational challenge in DAG learning, several
studies \citep{zheng2018dagswith,lee2019scaling,zhang2022truncated}
have shifted the combinatorial search to a continuous optimization
problem, proving to be more scalable and adaptable to different SEMs.
Specifically, researchers have introduced various smooth functions
to evaluate whether a directed adjacency matrix represents a DAG,
i.e., $\text{\ensuremath{\mathbf{A}} \ensuremath{\in\ }\ensuremath{\mathbb{D}\Leftrightarrow\ }h\ensuremath{\left(\mathbf{A}\right)}}=0$.
Consequently, solving Eq.~(\ref{eq:score}) can be accomplished through
a constrained optimization method, such as the augmented Lagrangian
method. Despite these advancements, existing methods still fall short
in ensuring a valid acyclic output, possibly requiring additional
post-processing steps \citep{buhlmann2014camcausal}. Therefore, we
aim to find a convenient representation for the DAG space that allows
direct optimization within it. This approach guarantees the output
of a DAG at any stage during the learning process.

A common approach for representing a DAG involves utilizing a permutation
matrix $\boldsymbol{\Pi}\in\{0,1\}^{d\times d}$ and an upper triangular
matrix $\mathbf{U}\in\{0,1\}^{d\times d}$, i.e., $\mathbf{A}=\boldsymbol{\Pi}^{T}\mathbf{U}\boldsymbol{\Pi}$.
This formulation arises from the inherent property of a DAG, where
there exists at least one valid permutation (or causal order) $\pi\in\mathbb{R}^{d}$,
implying that there is no direct edge from node $X_{\pi(j)}$ and
$X_{\pi(i)}$ if $\pi(i)<\pi(j)$. To seamlessly integrate this formula
into a continuous DAG learning framework, we can parameterize the
permutation matrix and the upper triangular matrix using Gumbel-Sinkhorn
\citep{mena2018learning} and Gumbel-Softmax \citep{jang2017categorical}
distributions, respectively. Yet, this approach introduces a high
complexity (i.e., $\mathcal{O}\left(d^{3}\right)$) due to the Hungarian
algorithm \citep{kuhn1955thehungarian} used in the forward pass of
Gumbel-Sinkhorn. 

Motivated by the same objective, \citep{yu2021dagswith} introduce
No-Curl, a simple DAG mapping function that facilitates a projection
of any weighted adjacency matrix from an arbitrary directed graph
to the DAG space through a straightforward process: applying element-wise
multiplication with the gradient of a priority scores vector associated
with the graph vertices. We will provide a detailed explanation of
the No-Curl characterization, as it constitutes a crucial component
of our proposed method.

Let $\mathbf{p}\in\mathbb{R}^{d}$ be a vector representing priority
scores of $d$ nodes in a directed graph $\mathcal{G}$. The gradient
of \textbf{$\mathbf{p}$}, denoted by ${\mathrm{grad}\left(\mathbf{p}\right):\mathbb{\mathbb{R}}^{d}\rightarrow\mathbb{R}^{d\times d}}$,
is defined as follows: 

\begin{equation}
\mathrm{grad}\left(\mathbf{p}\right){}_{ij}=\mathbf{p}_{j}-\mathbf{p}_{i}.\label{eq:grad}
\end{equation}

Then, No-Curl proposes mapping $\mathbf{p}$ and a weighted directed
adjacency matrix $\mathbf{W}$, denoted by $\gamma\left(\mathbf{W},\mathbf{p}\right):\mathbb{R}^{d\times d}\times\mathbb{R}^{d}\rightarrow\mathbb{R}^{d\times d}$,
to the DAG space by:
\begin{equation}
\gamma\left(\mathbf{W},\mathbf{p}\right)=\mathbf{W}\circ\mathrm{ReLU}\left(\mathrm{grad}\left(\mathbf{p}\right)\right),\label{eq:no-curl}
\end{equation}
where $\mathrm{ReLU}$ is the rectified linear unit activation function.
For a detailed proof of No-Curl, we refer to \citep{yu2021dagswith}.
Here, we provide an intuitive explanation of the mapping $\gamma\left(\mathbf{W},\mathbf{p}\right)$:
the priority scores vector $\mathbf{p}$ implicitly corresponds to
a causal order of $d$ nodes when we sort $\mathbf{p}$ in increasing
order. To avoid cycles, we exclusively permit edges from nodes with
lower scores to nodes with higher scores, i.e., $\mbox{\ensuremath{p_{i}<p_{j}\Rightarrow(i\rightarrow j)}}$,
which is equivalent to $\mbox{\ensuremath{\mathrm{grad}\left(\mathbf{p}\right)_{ij}>0}}$.
Hence, $\text{\mbox{\ensuremath{\mathrm{ReLU}\left(\mathrm{grad}\left(\mathbf{p}\right)\right)}}}$
eliminates cycle-inducing edges by zeroing out all $\mbox{\ensuremath{\mathrm{grad}\left(\mathbf{p}\right)_{ij}\leq0}}$.
Finally, we remove spurious edges through a Hadamard product with
a weighted adjacency matrix $\mathbf{W}$. 

Unfortunately, the No-Curl characterization is specifically crafted
for a weighted DAG representation. Indeed, we favor a binary representation,
since it offers greater flexibility to model both linear and non-linear
functional relationships. To address this preference, we propose a
novel adaptation of the No-Curl approach for representing binary adjacency
matrices of DAGs, as detailed in the next section. 

\section{Proposed method}

The main goal of this study is to propose a scalable and fully differentiable
framework to approximate the posterior distribution over DAGs given
observational data. To this end, we first introduce a novel measure to parameterize
the probabilistic model over DAGs by extending the No-Curl characterization.
Leveraging this probabilistic model for DAG sampling and integrating
it with the variational inference framework, we then introduce a new
Bayesian causal discovery method, named VCUDA (\textbf{\uline{V}}ariational
\textbf{\uline{C}}ausal Discovery \textbf{\uline{U}}nconstraine\textbf{\uline{d}}
by \textbf{\uline{A}}cyclicity), offering precise capture and effective generation of samples from the complex posterior distribution of DAGs.

\subsection{Differentiable DAG sampling\label{subsec:Differentiable-DAG-sampling}}

As discussed in the previous section, a weighted adjacency matrix
to represent a DAG does not align with our purpose. Therefore, we
extend the No-Curl characterization by substituting the $\textrm{ReLU(.)}$
by a tempered $\text{sigmoid(.)}$ function, allowing us to represent
a DAG using a binary adjacency matrix.
\begin{theorem}
\label{thm:1}Let $\mathbb{\mathbf{A}}\in\{0,1\}^{d\times d}$ be
an adjacency matrix of a graph of $d$ nodes. Then, $\mathbf{A}$
is DAG if and only if there exists a vector of priority
scores $\mathbf{p}\in\mathbb{R}^{d}$ and a corresponding binary matrix $\mathbf{W}\in\{0,1\}^{d\times d}$
such that: 
\[
\mathbf{A}=\nu\left(\mathbf{W},\mathbf{p}\right)=\mathbf{W}\circ\lim_{t\rightarrow0}\mathrm{sigmoid}\left(\frac{\mathrm{grad}(\mathbf{p})}{t}\right)
\]
 where $t>0$ is a strictly positive temperature and $\mathbf{p}$
contains no duplicate elements, i.e. , $p_{i}\neq p_{j}\forall i,j$.
\end{theorem}

\begin{proof}
See the Appendix \cite{hoang2024scalable} for more details. 
\end{proof}
Closely related to our method, COSMO \citep{massidda2023constraintfree}
also introduces a smooth orientation matrix for unconstrained DAG
learning. It is crucial to highlight that our study is motivated by
distinct objectives, specifically, addressing challenges related to
scalability and generalization in Bayesian causal discovery. This
convergence in ideas reaffirms the significance of our approach in
independently tackling common challenges, underscoring its broader
applicability and relevance. 

Based on Theorem~\ref{thm:1}, we further introduce a new probabilistic
model over DAGs space as follows: 
\begin{align}
P(\mathbf{A}) & =\sum_{\mathbf{W}}\int_{\mathbf{p}}P\left(\mathbf{W},\mathbf{p}\right)d\mathbf{p}\nonumber \\
\text{s.t. }\mathbf{A} & =\mathbf{W}\circ\lim_{t\rightarrow0}\mathrm{sigmoid}\left(\frac{\mathrm{grad}(\mathbf{p})}{t}\right)\label{eq:prob_model}
\end{align}
where $P(\mathbf{W})$ and $P(\mathbf{p})$ are distributions over
edges and priority scores, respectively. As a result, Eq. (\ref{eq:prob_model})
provides a fast and assured sampling approach of DAGs without evaluating
any explicit acyclicity constraints. We follow \citep{charpentier2022differentiable},
utilizing the Gumbel-Softmax \citep{jang2017categorical} to model
the discrete distribution over edges. Let $\varphi\in[0,1]$ be the
probability of the existence of an edge from node $X_{i}$ to $X_{j}$.
The Gumbel-Softmax, which is a continuous distribution, enables a
differentiable approximation of samples from a discrete distribution,
e.g., Bernoulli distribution: $\mbox{\ensuremath{\hat{W}_{ij}\in\left[0,1\right]\sim\textrm{Gumbel-Softmax}_{\tau}(\varphi_{ij})}}$,
where $\mbox{\ensuremath{\tau>0}}$ is the temperature parameter controlling
the smoothness of the categorical during sampling. For example, a
low value of $\tau$ generates more likely one-hot encoding samples,
making the Gumbel-Softmax distribution resemble the original categorical
distribution. Consequently, we can directly generate a DAG by sampling
an edge matrix from a Gumbel-Softmax distribution $\mbox{\ensuremath{\mathbf{W}\sim\textrm{Gumbel-Softmax}(\boldsymbol{\varphi})}}$,
and a priority score vector from a multivariate distribution $\mbox{\ensuremath{\mathbf{p}\sim P_{\psi}(\mathbf{p})}}$:
$\mbox{\ensuremath{\mathbf{A}\sim P_{\varphi,\psi}\left(\mathbf{A}\right)}}$
where $\varphi$ and $\psi$ are parameters defined distributions
of $\mathbf{W}$ and $\mathbf{p}$, respectively. 

\textbf{Computational complexity:} Our proposed probabilistic model
significantly speeds up the DAG sampling time compared with related
studies using the Gumbel-Sinkhorn approach, such as BCD-nets \citep{cundy2021bcdnets2}.
To elaborate, our proposed approach requires $\mathcal{O}(d^{2})$
for sampling the edge matrix $\mathbf{W}$ and $\mathcal{O}(d)$ for
sampling the priority scores vector $\mathbf{p}$. This leads to an
overall computational complexity of $\mathcal{O}(d^{2})$. A closely
related study to our approach is BayesDAG \citep{annadani2023bayesdag},
which suggests replacing $\text{ReLU}(.)$ with $\text{Step}(.)$.
However, their approach encounters the problem of uninformative gradients
due to the intrinsic property of $\text{Step}(.)$. Consequently,
\citep{annadani2023bayesdag} still utilizes the Gumbel-Sinkhorn operator
to approximate the distribution over permutation matrices, incurring
high complexity at $\mathcal{O}(d^{3})$. 

\RestyleAlgo{ruled}

\begin{algorithm}[t]
\caption{VCUDA (\textbf{\uline{V}}ariational \textbf{\uline{C}}ausal
Discovery \textbf{\uline{U}}nconstraine\textbf{\uline{d}} by
\textbf{\uline{A}}cyclicity) }
\KwIn{Observational dataset $\mathbf{X}$; prior distributions $P_\text{prior}(\mathbf{W})$, $P_\text{prior}(\mathbf{p})$; temperature $t$, regularizers $\lambda_1$, $\lambda_2$, training iterations $T$}  
Initialize parameters $\varphi$, $\psi$, $\theta$\; \For{$i = 1 \ldots T$}{     \For{$\mathbf{X}_\text{batch} \in \mathbf{X}$}{         Sample $\mathbf{W} \sim P_\varphi(\mathbf{W})$\;         Sample $\mathbf{p} \sim P_\psi(\mathbf{p})$\;         Compute $\mathbf{A} = \mathbf{W} \circ  \text{sigmoid}\left(\frac{\mathrm{grad}(\mathbf{p})}{t}\right)$\;         Compute $\hat{\mathbf{X}} = f_\theta\left(\mathbf{X}_\text{batch}, \mathbf{A}\right)$\;         Maximize ELBO loss (Eq.  (\ref{eq:final_loss})) w.r.t $\varphi$, $\psi$, $\theta$\;     } } 
\KwRet{$\varphi$, $\psi$, $\theta$}
\label{alg:vi-training}
\end{algorithm}

\subsection{Variational Inference DAG Learning \label{subsec:Variational-Inference-DAG}}

As mentioned earlier, Bayesian causal structure learning aims to determine
the posterior distribution over DAGs. However, directly computing
the posterior becomes infeasible due to the intractability of the
marginal data distribution. To address this challenge, we turn to
variational inference. Here, we leverage the probabilistic DAG model
$P_{\varphi,\psi}(\mathbf{A})$ introduced in Section~\ref{subsec:Differentiable-DAG-sampling}
to approximate the true posterior distribution $P\left(\mathbf{A}\mid\mathbf{X}\right)$.
In essence, we aim to optimize the variational parameters to minimize
the KL divergence between the approximate and true posterior distributions
that are equivalent to maximizing the evidence lower bound (ELBO)
objective as follows: 
\begin{align}
\mathrm{\max_{\theta,\varphi,\psi}}\mathcal{L} & =\underbrace{\mathbb{E}_{\mathbf{W,}\mathbf{p}\sim P_{\varphi,\psi}\left(\mathbf{W},\mathbf{p}\right)}\left[\log P_{\theta}\left(\mathbf{X}\mid\mathbf{\mathbf{W}},\text{\ensuremath{\mathbf{p}}}\right)\right]}_{\mathrm{(i)}}\label{eq:elbo}\\
 & -\underbrace{D_{\mathrm{KL}}\left(P_{\varphi,\psi}\left(\mathbf{W,\mathbf{p}}\right)\parallel P_{\mathrm{prior}}\left(\mathbf{W},\mathbf{p}\right)\right)}_{\mathrm{(ii)}}\nonumber.
\end{align}

The objective in Eq.~(\ref{eq:elbo}) consists of two
terms: i) the first is the log-likelihood of the data given the causal
structure model and ii) the second is the KL divergence between the
approximate posterior distribution and the prior distribution. With appropriate choices of variational families and prior models,
the optimized parameters $\theta,\varphi,\psi$ from Eq.~(\ref{eq:elbo}) minimize the divergence between the approximate distribution and the true
distribution, i.e., $P_{\varphi,\psi}(\mathbf{A}) \approx P(\mathbf{A}\mid\mathbf{X})$.

To compute $\textrm{(i)}$, we begin by sampling a DAG adjacency matrix
$\mathbf{A}$ from the approximate distribution in each iteration.
For every node $X_{i}$, we reconstruct its values by applying masking
on the observed data $\mathbf{X}$ with the sampled $\mathbf{A}$.
This is then followed by a transformation $f_{i,\theta}$ , parameterized
using neural networks:
\begin{equation}
\hat{X_{i}}=f_{i,\theta}(A_{i}\circ\mathbf{X}),
\end{equation}
where $A_{i}$ is the $i^{\text{th}}$ column in the adjacency matrix
$\mathbf{A}$. By assuming that the data has a Gaussian distribution
with unit variance, we can approximate the first term by the least
square loss, i.e., $\parallel\mathbf{X}-\mathbf{\hat{\mathbf{X}}}\parallel^{2}$. 

To compute $\textrm{(ii)}$, we initially employ a mean-field factorization
for the variational model, i.e., ${\ensuremath{P_{\varphi,\psi}\left(\mathbf{W},\mathbf{p}\right)=P_{\varphi}\left(\mathbf{W}\right)P_{\psi}\left(\mathbf{p}\right)}}$.
This mean-field factorization provides us a convenient way to calculate
the KL divergence, represented as: 

{\footnotesize{}
\begin{align}
 & D_{KL}\left(P_{\varphi,\psi}\left(\mathbf{W,\mathbf{p}}\right)\parallel P_{\mathrm{prior}}\left(\mathbf{W},\mathbf{p}\right)\right)\\
=\: & \mathbb{E}_{\mathbf{p},\mathbf{W}\sim P_{\varphi,\psi}\left(\mathbf{p},\mathbf{W}\right)}\left[\log\frac{P_{\varphi}\left(\mathbf{W}\right)P_{\mathbf{\psi}}\left(\mathbf{p}\right)}{P_{\mathrm{prior}}\left(\mathbf{W}\right)P_{\mathrm{prior}}\left(\mathbf{p}\right)}\right]\\
=\: & \mathbb{E}_{\mathbf{p},\mathbf{W}\sim P_{\varphi,\psi}\left(\mathbf{p},\mathbf{W}\right)}\left[\log\frac{P_{\varphi}\left(\mathbf{W}\right)}{P_{\mathrm{prior}}\left(\mathbf{W}\right)}+\log\frac{P_{\mathbf{\psi}}\left(\mathbf{p}\right)}{P_{\mathrm{prior}}\left(\mathbf{p}\right)}\right]\\
=\: & \int\int P_{\varphi}\left(\mathbf{W}\right)P_{\mathbf{\psi}}\left(\mathbf{p}\right)\left[\log\frac{P_{\varphi}\left(\mathbf{W}\right)}{P_{\mathrm{prior}}\left(\mathbf{W}\right)}+\log\frac{P_{\mathbf{\psi}}\left(\mathbf{p}\right)}{P_{\mathrm{prior}}\left(\mathbf{p}\right)}\right]d\mathbf{W}d\mathbf{p}\\
=\: & \int\int P_{\varphi}\left(\mathbf{W}\right)P_{\mathbf{\psi}}\left(\mathbf{p}\right)\left[\log\frac{P_{\varphi}\left(\mathbf{W}\right)}{P_{\mathrm{prior}}\left(\mathbf{W}\right)}\right]d\mathbf{W}d\mathbf{p}\\
 & +\int\int P_{\varphi}\left(\mathbf{W}\right)P_{\mathbf{\psi}}\left(\mathbf{p}\right)\left[\log\frac{P_{\mathbf{\psi}}\left(\mathbf{p}\right)}{P_{\mathrm{prior}}\left(\mathbf{p}\right)}\right]d\mathbf{W}d\mathbf{p}\\
=\: & \int P_{\mathbf{\psi}}\left(\mathbf{p}\right)D_{KL}\left(P_{\varphi}\left(\mathbf{W}\right)\parallel P_{\mathrm{prior}}\left(\mathbf{W}\right)\right)d\mathbf{p}\\
 & +\int P_{\varphi}\left(\mathbf{W}\right)D_{KL}\left(P_{\mathbf{\psi}}\left(\mathbf{p}\right)\parallel P_{\mathrm{prior}}\left(\mathbf{p}\right)\right)d\mathbf{W}\\
=\: & D_{KL}\left(P_{\varphi}\left(\mathbf{W}\right)\parallel P_{\mathrm{prior}}\left(\mathbf{W}\right)\right)+D_{KL}\left(P_{\mathbf{\psi}}\left(\mathbf{p}\right)\parallel P_{\mathrm{prior}}\left(\mathbf{p}\right)\right)
\end{align}}

Consequently, we can compute
the KL divergence between the approximate posterior distribution and
the prior distribution over DAGs via the sum of the KL divergence
between the variational model and the prior distribution over the
edge matrix $\mathbf{W}$ and the priority scores vector $\mathbf{p}$.

\textbf{Variational Families: }For the distribution over the priority
scores vector $\mathbf{p}$, we opt for the isotropic Gaussian, i.e.,
$\mathbf{p}\sim\mathcal{N}(\boldsymbol{\mu},\sigma^{2}\mathbf{I})$.
As discussed earlier, we choose the Gumbel-Softmax distribution \citep{jang2017categorical}
for the variational model of the edge matrix $\mathbf{W}$. These
choices enable us to utilize the pathwise gradient, offering a lower
variance approach compared to the score-function method \citep{mohamed2020montecarlo}.
To compute the gradient, we leverage the straight-through estimator
\citep{bengio2013estimating}: for $\mathbf{p}$, we use the rounded
value of $\mathrm{sigmoid}\left(\frac{\mathrm{grad}(\mathbf{p})}{t}\right)$
in the forward pass, and its continuous value in the backward pass.
For $\mathbf{W}$, we use the discrete value $W_{ij}=\argmax\left[1-\hat{W}_{ij},\hat{W}_{ij}\right]$
in the forward pass, and the continuous approximation $\hat{W}_{ij}$
in the backward pass. 

\textbf{Prior Distribution: }A well-chosen prior encapsulates existing
knowledge about the model parameters, thereby guiding the inference
process. In line with the belief in the sparsity of causal DAGs, we
set a small prior $P_{\mathrm{prior}}\left(W_{ij}\right)$ on the
edge probability. For an effective gradient estimation, we define
the prior distribution of the priority scores vector as a normal distribution
with a mean of zero and a small variance. 

Incorporating all the above design choices, the final loss can be
expressed as follows: 
\begin{align}
\mathrm{\max_{\theta,\varphi,\psi}}\mathcal{L} & =-\sum\left(X_{ij}-\hat{X}_{ij}\right)^{2}\label{eq:final_loss}\\
 & -D_{KL}\left(P_{\varphi}\left(\mathbf{W}\right)\parallel P_{\mathrm{prior}}\left(\mathbf{W}\right)\right)\nonumber \\
 & -D_{KL}\left(P_{\mathbf{\psi}}\left(\mathbf{p}\right)\parallel P_{\mathrm{prior}}\left(\mathbf{p}\right)\right),\nonumber 
\end{align}
where $\hat{X_{i}}=f_{i,\theta}(\mathbf{A}_{i}\circ\mathbf{X})$ is
the reconstructed data from the sampled DAG. In the implementation, we
divide the total loss by the number of nodes $d$ for stable numerical
optimization. The training process of the proposed approach are summarized
in Algorithm~\ref{alg:vi-training}.

\begin{figure*}[tbh]
\begin{minipage}[t]{0.33\textwidth}%
\subfloat{\includegraphics[width=1.4\textwidth]{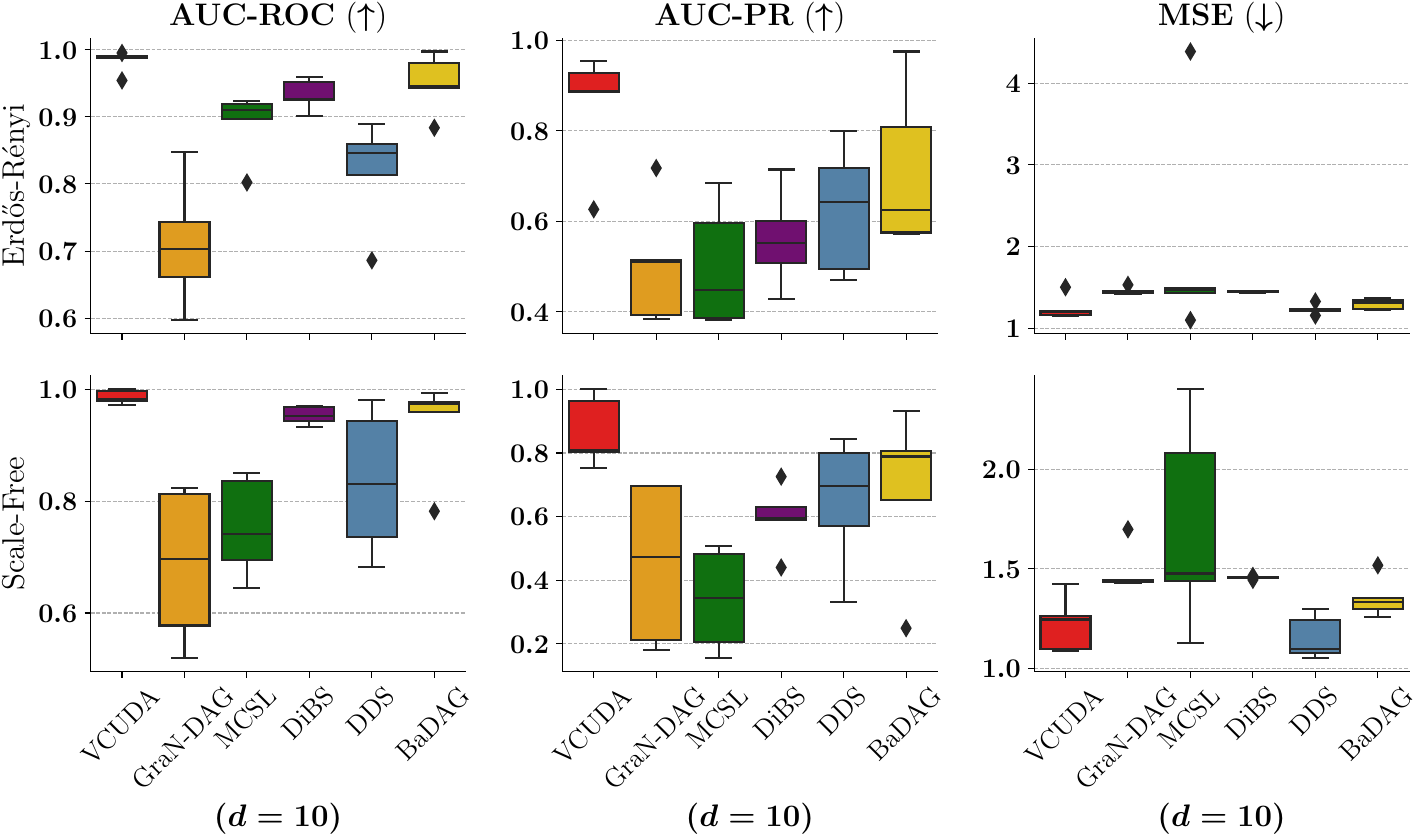}}%
\end{minipage}\hspace{3cm}%
\begin{minipage}[t]{0.33\textwidth}%
\subfloat{\includegraphics[width=1.4\textwidth]{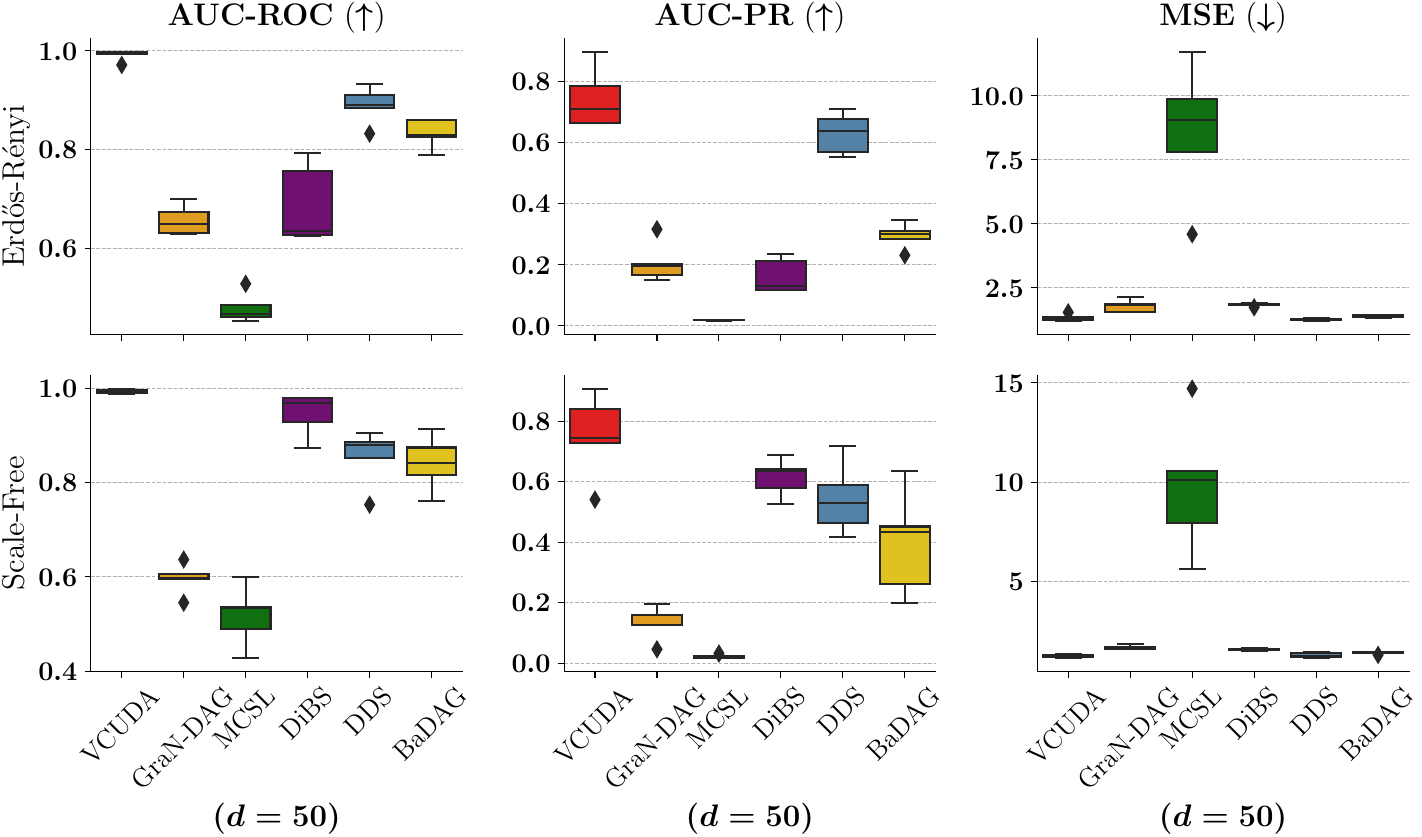}}%
\end{minipage}
\caption{Performance on synthetic data generated
from linear Gaussian models with $d=10$ and $d=50$ variables of different graph models. The reported values are aggregated from 10 independent runs. VCUDA achieves the best results across most metrics and outperforms other Bayesian approaches (DiBS and DDS). $\downarrow$
denotes lower is better and $\uparrow$ denotes higher is better.}
\label{fig:linear-synthetic}
\end{figure*}

\begin{figure*}
\begin{minipage}[t]{0.33\textwidth}%
\subfloat{\includegraphics[width=1.4\textwidth]{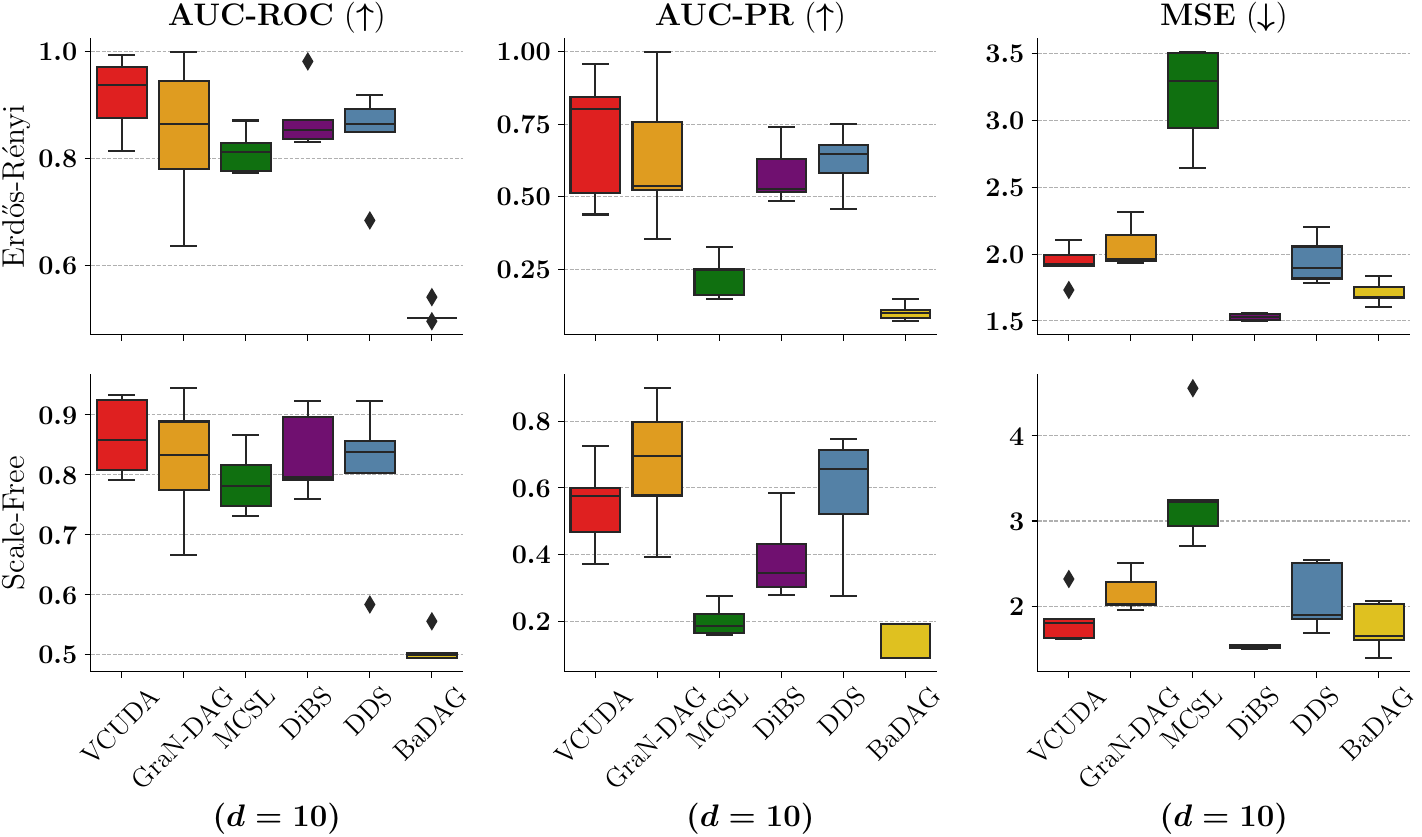}}%
\end{minipage}\hspace{3cm}%
\begin{minipage}[t]{0.33\textwidth}%
\subfloat{\includegraphics[width=1.4\textwidth]{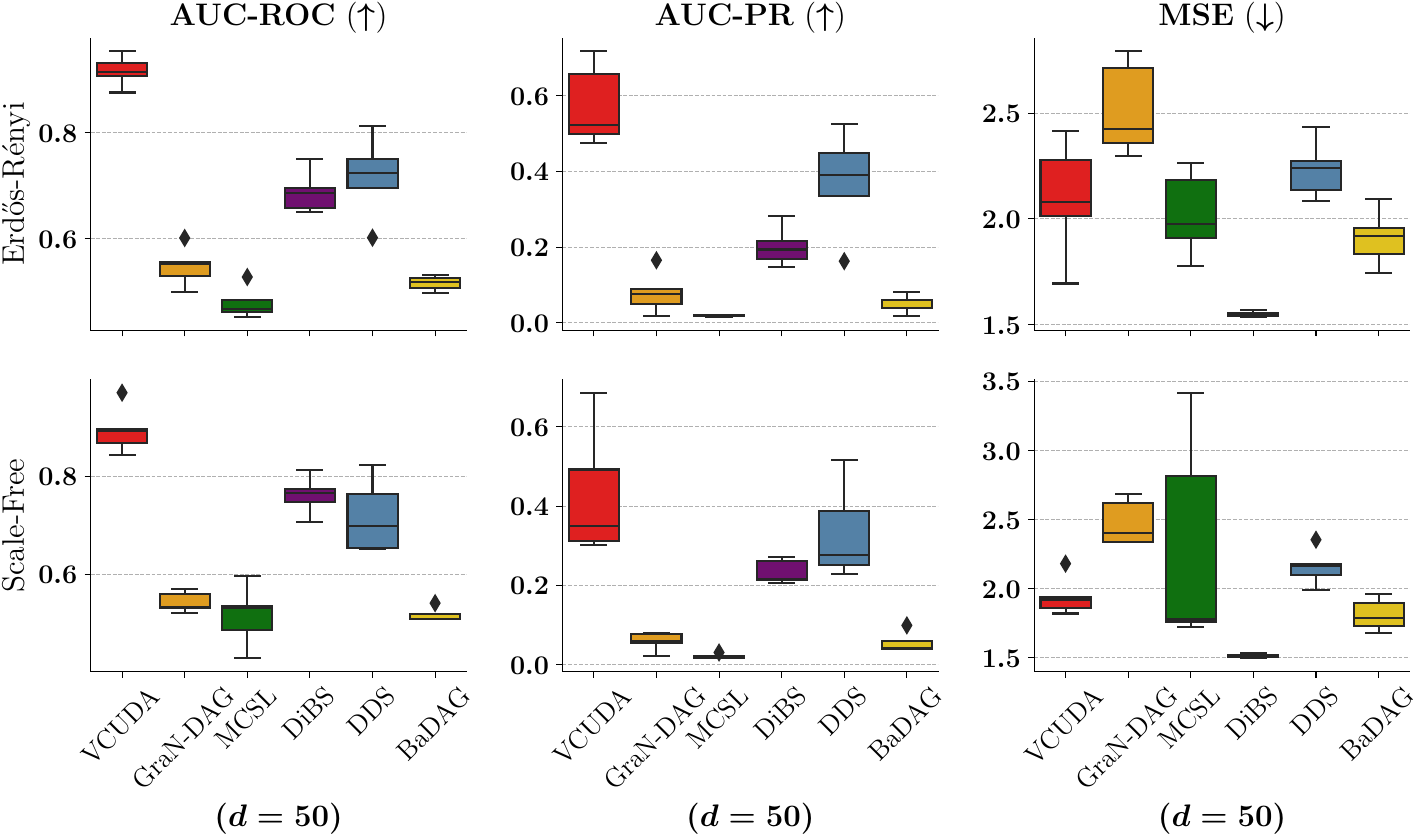}}%
\end{minipage}
\caption{Performance on synthetic data generated
from nonlinear Gaussian models with $d=10$ and $d=50$ variables of different graph models. The reported values are aggregated from 10 independent runs. Our proposed approach VCUDA achieves the best results across most metrics and outperforms other Bayesian based approaches
(DiBS and DDS). $\downarrow$ denotes lower is better and $\uparrow$ denotes higher is better.}
\label{fig:nonlinear-synthetic}
\end{figure*}

\section{Experiment\label{sec:Experiment}}

In this section, we demonstrate extensive experiments showing the
empirical performance of our proposed method on DAG sampling and DAG
structure learning tasks. 


\textbf{Baselines. }Regarding DAG sampling, we compare the proposed
method with the Gumbel-Sinkhorn and Gumbel-Top-k approaches \citep{charpentier2022differentiable}
in terms of sampling time for thousands of variables. Unlike our method, which samples a priority score vector, these approaches focus on sampling a permutation matrix. The Gumbel-Sinkhorn method \cite{mena2018learning} leverages the Sinkhorn operator to approximate the distribution over the permutation matrix. Meanwhile, Gumbel-Top-k combines the Gumbel-Top-k distribution \cite{kool2019stochastic} with the Soft-Sort operator \cite{prillo2020softsort} to achieve faster sampling compared to Gumbel-Sinkhorn. Regarding DAG structure learning, we focus our comparison on differentiable
methods and hence select five state-of-the-art causal discovery methods that belong to both point-estimations and Bayesian based baselines:
\begin{itemize}
\item \textbf{GraN-DAG} \citep{lachapelle2020gradientbased} utilizes the product
of neural network computation path as a proxy for the adjacency matrix and NOTEARS \citep{zheng2018dagswith}
for the DAG constraint.\footnote{\url{https://github.com/kurowasan/GraN-DAG}}

\item \textbf{Masked-DAG} \citep{ng2022maskedgradientbased} leverages Gumbel-Sigmoid
to parameterize the binary adjacency matrix and NOTEARS \citep{zheng2018dagswith}
to impose the DAG constraint. \footnote{\url{https://github.com/huawei-noah/trustworthyAI}}
 
\item \textbf{DiBS} \citep{lorch2021dibsdifferentiable} models the Bayesian causal
structure problem from the latent space of a probabilistic graph representation
and employs Stein variational gradient descent to solve the problem.
The study also exploits NOTEARS \citep{zheng2018dagswith} to impose the
DAG constraint on the latent space.\footnote{\url{https://github.com/larslorch/dibs}}

\item \textbf{DDS} \citep{charpentier2022differentiable} introduces a differentiable
DAG sampling via sampling an edge matrix and a permutation matrix, which is integrated with a variational inference
model to solve Bayesian causal structure learning. \footnote{\url{https://github.com/sharpenb/Differentiable-DAG-Sampling}}

\item \textbf{BaDAG} \citep{annadani2023bayesdag} combines both stochastic gradient Markov Chain Monte Carlo and variational inference for Bayesian causal discovery. The study leverages No-Curl constraint and Sinkhorn algorithm to sample DAGs. \footnote{\url{https://github.com/microsoft/Project-BayesDAG}}

\end{itemize}
We obtain the original implementations and the recommended hyper-parameters
for these baselines. 

\textbf{Datasets. }We benchmark these methods on both synthetic and
real datasets. For synthetic datasets, we closely follow \citep{lachapelle2020gradientbased,ng2022maskedgradientbased,charpentier2022differentiable}.
For generating causal DAGs, we consider Erd\H{o}s-Rényi (ER) and scale-free
(SF) network models with average degree equal to 1. We vary the graph
size in terms of number of nodes $d=\{10,50,100\}$ and consider both
linear and nonlinear Gaussian SEMs. For linear model, we generate
a weighted adjacency matrix $\mathbf{W}\in\mathbb{R}^{d\times d}$
with edges' weights randomly sampled from $\mathcal{U}\left(\left[-2,-0.5\right]\cup\left[0.5,2\right]\right)$.
We then generate the data $\mathbf{X}\in\mathbb{R}^{n\times d}$ following
the linear SEM: $\mathbf{X}=\mathbf{W}^{T}\mathbf{X}+\epsilon$, where
$\epsilon\sim\mathcal{N}\left(0,1\right)$. For nonlinear model, we
generate the data following $X_{i}=f_{i}\left(X_{\mathrm{pa}(i)}\right)+\epsilon_{i}$,
where the functional model $f_{i}$ is generated from Gaussian Process
with RBF kernel of bandwidth one and $\epsilon_{i}\sim\mathcal{N}\left(0,1\right)$.
In our experiments, we sample 10 datasets per setting where each dataset
includes a ground truth of the causal DAG's adjacency matrix, a training
dataset of 1,000 samples and a held out testing dataset of 100 samples. For real datasets, we closely
follow \citep{yu2021dagswith,lorch2021dibsdifferentiable,ng2022maskedgradientbased}.
We use Sachs dataset \citep{sachs2005causalproteinsignaling} which
measures the expression level of different proteins and phospholipids
in human cells. The data contains 853 observational samples generated
from a protein interaction network of 11 nodes and 17 edges. 

\textbf{Evaluation metrics. }We use the area under the curve of precision-recall
(AUC-PR) and the area under the receiver operating characteristic
curve (AUC-ROC) between the ground-truth binary adjacency matrix and
the output scores matrix, denoted as $\mathbf{S}$, where $S_{ij}$
represents the possibility of the presence of an edge from $X_{i}$
to $X_{j}$. For point estimation methods, we get the scores matrix
from the output before thresholding. For Bayesian-based methods, we
get the scores matrix by averaging 100 sampled binary adjacency matrix
from the learned probabilistic DAG model. We also evaluate the learned
functional model $f_{\theta}\left(.\right)$ by computing the mean
squared error (MSE) between the ground-truth node value $X_{i}$ and
the estimated node values $\hat{X}_{i}=f_{i,\theta}(\mathbf{A_{i}}\circ\mathbf{X})$
on a held-out dataset.

\textbf{Hyperparameters.} We use a neural network with one hidden
layer and ReLU activation to parameterize the functional models $f_{i}$ in the nonlinear setting and real-world dataset. We perform a random search
over the learning rate $lr\in\left[10^{-3},10^{-1}\right]$. The prior edge probability and scale of the priority scores vector are set to 0.01 and 0.1, respectively. Based on our ablation study (refer to Section \ref{sec:abl}), a temperature of 0.3 is chosen for its benefits. To prevent overfiting, VCUDA is trained by the Adam optimizer with the l2-regularization of $1e-4$. Furthermore, early stopping is employed with validation loss checks every 10 epochs. We find that the model is convergent within 500 epochs. 

\begin{figure*}[t]
\begin{minipage}[t]{0.33\textwidth}%
\subfloat[Linear model]{\includegraphics[width=1.4\textwidth]{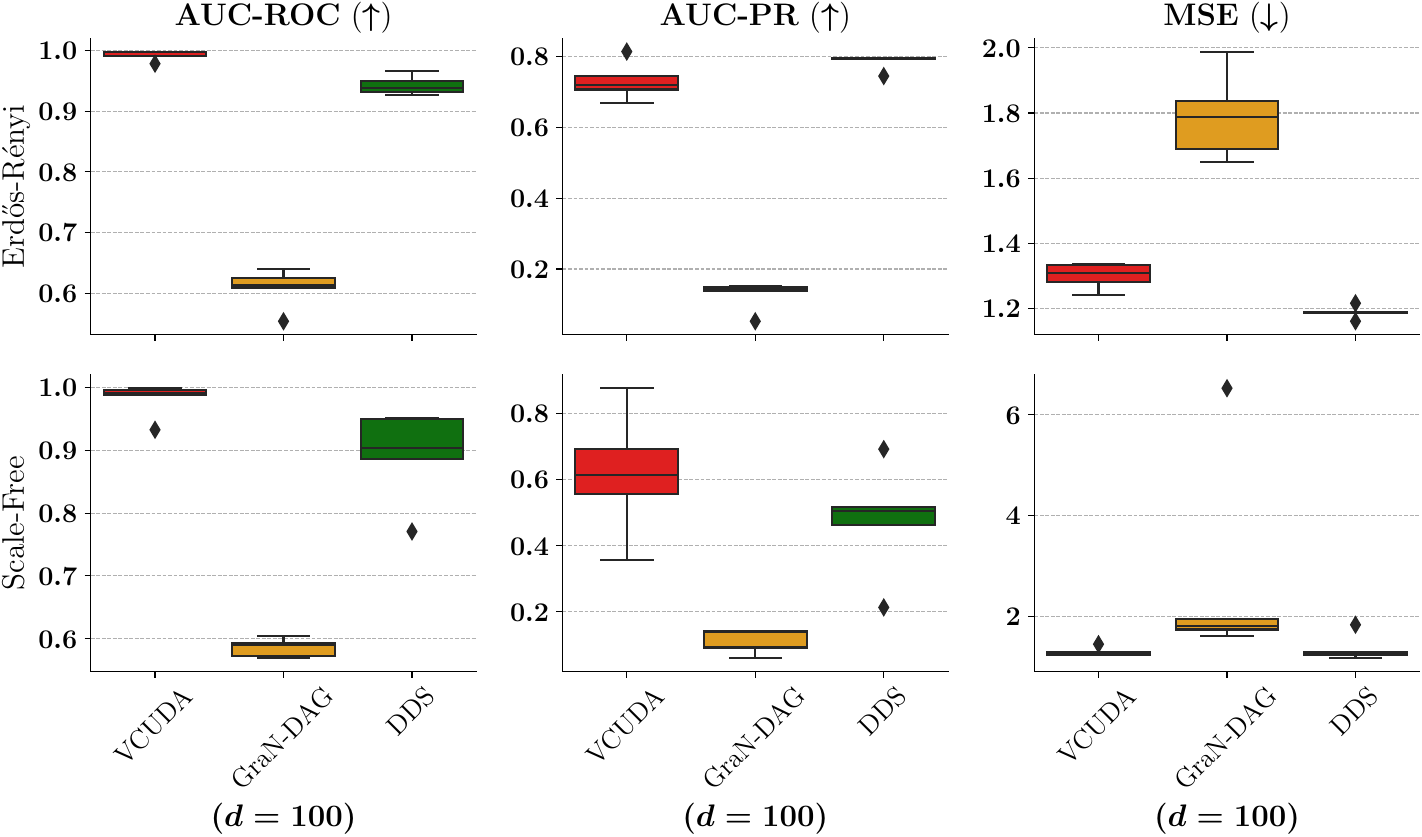}}%
\end{minipage}\hspace{3cm}%
\begin{minipage}[t]{0.33\textwidth}%
\subfloat[Nonlinear model]{\includegraphics[width=1.4\textwidth]{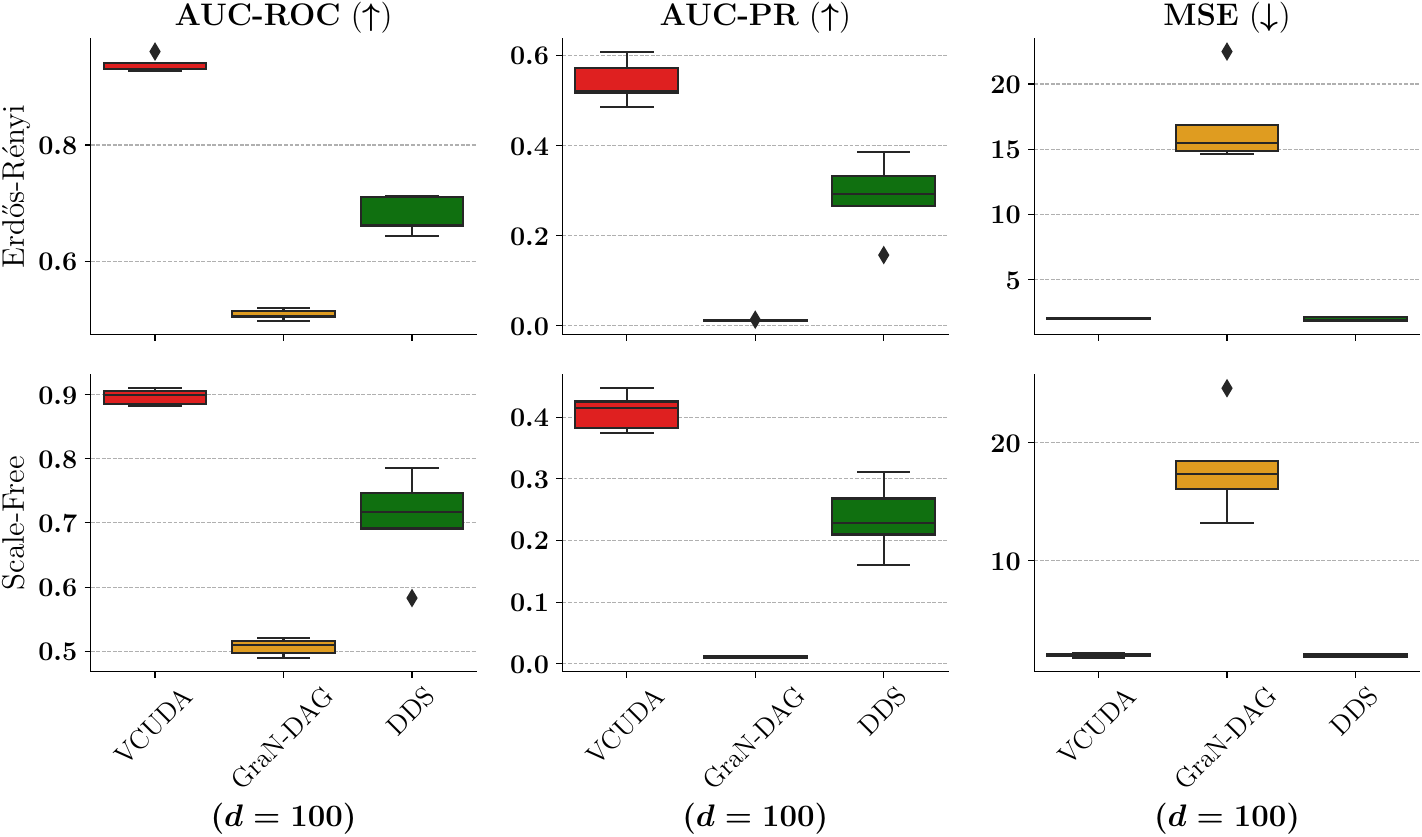}}%
\end{minipage}
\caption{Performance on high dimensional data with $d=100$ for different graphs and causal functional models. The reported values are aggregated from 10 independent runs. Our proposed approach VCUDA achieves the best results across most metrics. $\downarrow$ denotes lower is better and $\uparrow$ denotes higher is better.\newline}
\label{fig:hidimen}
\end{figure*}

\subsection{DAG sampling}

We compare our DAG sampling model to two well-known
models, including Gumbel-Sinkhorn and Gumbel-Top-k \citep{charpentier2022differentiable} on large-scale DAGs sampling. The results are shown in the Appendix \cite{hoang2024scalable}, indicating superiority in running time of our proposed model compared to the others, paving a road for scalable Bayesian causal discovery. More importantly, our sampling method achieves a consistently high performance when integrated into the variational framework to infer the DAGs' posterior distribution, which will be shown in the following section. 

%
\subsection{DAG structure learning}

\textbf{Synthetic datasets.} We present the results of DAG structure
learning on ER/SF graphs with $d=\{10,50\}$ for both linear and nonlinear
models in Figure~\ref{fig:linear-synthetic} and Figure~\ref{fig:nonlinear-synthetic},
respectively. The results exhibit the superior performance of VCUDA
across all settings, particularly in terms of AUC-ROC. Specifically,
the AUC-ROC values of VCUDA remain consistently high, always surpassing
0.9 for linear models and 0.8 for nonlinear models. In contrast, the
AUC-ROC values of the other baselines show volatility depending on
the settings. Additionally, VCUDA achieves a low MSE measures, outperforming
most baselines including GranDAG, MCSL, and DDS across all settings.
In comparison to DiBS, the MSE measures of VCUDA are comparable for
both linear and nonlinear settings. We observe a notably superior
performance of Bayesian baselines compared to point-estimation baselines,
especially in scenarios with higher dimensions (e.g., $d=50$). This
disparity in performance can be attributed to the Bayesian methods'
capacity to capture the uncertainty effectively. For more results on denser graphs, we refer to the Appendix \cite{hoang2024scalable}.

Furthermore, we study the performance of VCUDA and baseline methods on the high
dimensional causal graph with $d=100$ for both linear and nonlinear
models. We find that DiBS is computationally excessive producing an
error in the benchmarking device, while the running time results of MCSL and BaDAG exceeded
our defined time limit of 1 hour for each dataset. Therefore, Figure~\ref{fig:hidimen}
visualizes the results of VCUDA, GraNDAG, and DDS. Among these models, VCUDA shows a
consistent outperformance compared to other baselines despite the
challenge of high dimensional problems. 


\begin{table}
\begin{centering}
\caption{Performance on a real dataset of the protein signaling network. We report the mean $\pm$ std of AUC-ROC and AUC-PR metrics. Results are averaged over 10 different restarts. $\uparrow$ denotes higher is better. We highlight in bold the \textbf{best} result and in underline the \underline{second best} result. VCUDA achieves the best AUC-ROC and the second best AUC-PR.}
\label{tab:real-datasets}
\begin{tabular}{ccc}
\toprule 
Method & AUC-ROC $\left(\uparrow\right)$ & AUC-PR $\left(\uparrow\right)$\tabularnewline
\midrule
GraNDAG & $0.57\pm{\scriptstyle 0.02}$ & $0.26\pm{\scriptstyle 0.03}$\tabularnewline
MCSL    & $0.58\pm{\scriptstyle 0.04}$ & $0.21\pm{\scriptstyle 0.03}$\tabularnewline
DiBS    & $\underline{0.66\pm{\scriptstyle 0.03}}$ & $\mathbf{0.34\pm{\scriptstyle 0.07}}$ \tabularnewline
DDS     & $0.43\pm{\scriptstyle 0.02}$ & $0.16\pm{\scriptstyle 0.04}$\tabularnewline
BaDAG     & $0.48\pm{\scriptstyle 0.01}$ & $0.17\pm{\scriptstyle 0.02}$\tabularnewline
\midrule 
VCUDA (Ours)   & $\mathbf{0.71\pm{\scriptstyle 0.04}}$ & $\underline{0.32\pm{\scriptstyle 0.05}}$\tabularnewline
\bottomrule
\end{tabular}
\par\end{centering}
\end{table}

\textbf{Real datasets. }Table~\ref{tab:real-datasets} displays the
AUC-ROC and AUC-PR metrics on the real dataset of the protein signaling
network. It is crucial to acknowledge a notable model misspecification
in real-world data, given that the data might not adhere to the additive
noise model assumption. Nevertheless, VCUDA achieves the best AUC-ROC
and the second-best AUC-PR, which underscores the adaptability of
VCUDA in navigating the intricate of real-world scenarios. 

\begin{figure}[t]
\includegraphics[width=1\columnwidth]{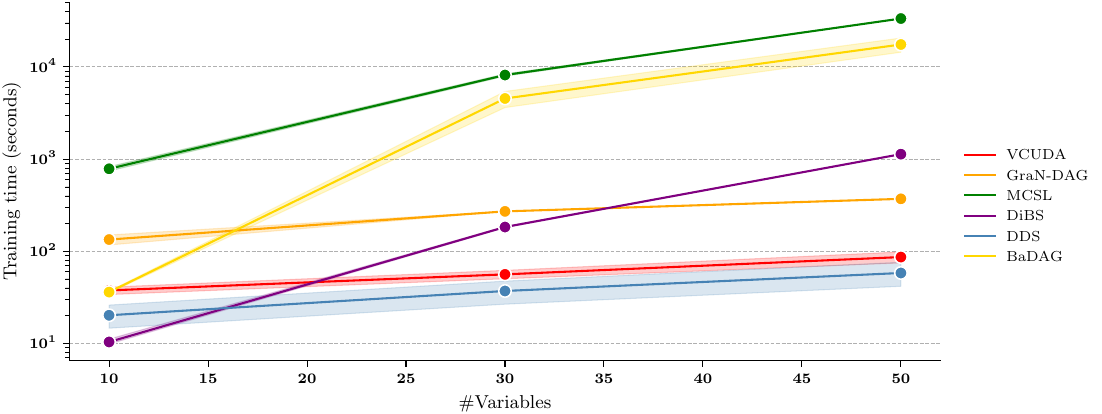}
\caption{The running time for causal discovery on synthetic datasets generated from a nonlinear model and ER graphs with ${d=[10,30,50]}$. VCUDA runs faster than 3 of 4 baselines, especially in high dimensions. \newline \vspace{2mm}}
\label{fig:training-time}
\end{figure}

\begin{figure}[h]
\includegraphics[width=1\columnwidth]{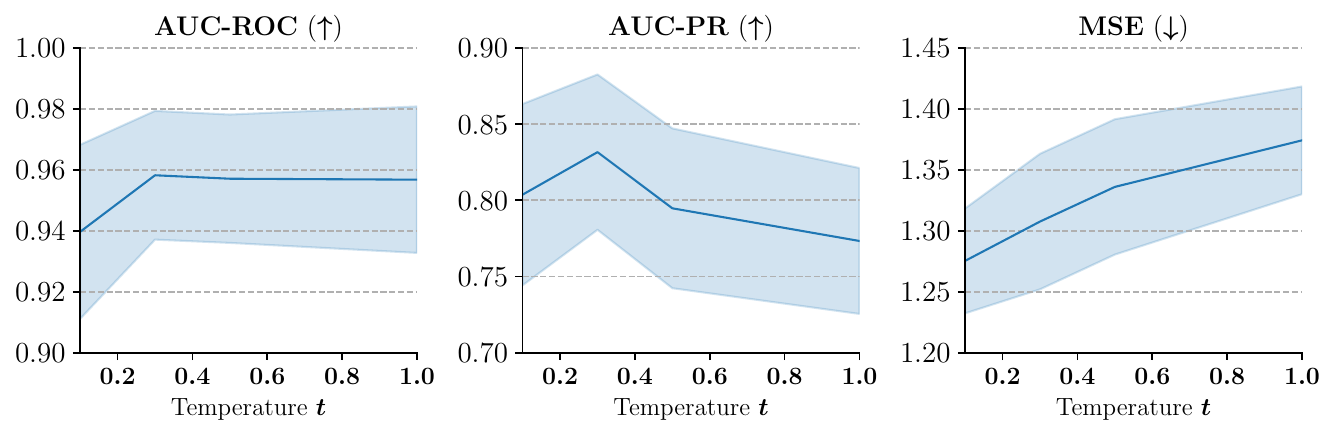}
\caption{The performance of VCUDA on different values of temperature $t$. The numerical results are obtained from 10 random datasets generated from a linear SEM with both ER and SF graph models.\newline\vspace{2mm}}
\label{fig:ablation_temp}
\end{figure}

\textbf{Running time.} We assess the running times of all methods in the ER-nonlinear setting
with varying the number of nodes. The results are presented in Figure \ref{fig:training-time}.
The results reveal that VCUDA demonstrates significantly faster running
times compared to GraN-DAG, MCSL, and DiBS, particularly as the number
of nodes increases. This observation can be explained by the additional
computational burden imposed on these methods due to the need to assess
the DAG-ness constraint (e.g., No-TEARS) throughout the process. 
While VCUDA's runtime is competitive with DDS using Gumbel-Top-k for DAG sampling, it achieves a substantial improvement in AUC-ROC, highlighting its efficacy in capturing the underlying causal structure. This translates to VCUDA being a superior solution with negligible runtime increase. Notably, even with more training iterations, DDS fails to significantly improve its overall performance.

\textbf{Ablation Study. }\label{sec:abl} We investigate the impact of temperature $t$ on the performance of VCUDA by evaluating different values within the range \{0.1, 0.3, 0.5, 1.0\}. As shown in Figure \ref{fig:ablation_temp}, lower temperatures can improve VCUDA's performance metrics, including AUC-ROC, AUC-PR, and MSE. Based on these results, we chose a temperature of 0.3 for all the experiments due to its consistent performance.

\section{Conclusion}

We introduce VCUDA, a scalable approach for Bayesian causal discovery
from observational data. By eliminating explicit acyclicity constraints,
we propose a differentiable approach for DAGs sampling, enabling fast
generation of large DAGs within seconds. In addition, the efficient
sampling model enhances Bayesian inference for causal structure models
when integrated into the variational inference framework. Extensive
experiments on synthetic and real datasets showcase VCUDA's superior
performance, outpacing other baselines in terms of multiple metrics,
all achieved with remarkable efficiency.



\bibliography{references}

\onecolumn
\renewcommand\thesubsection{\Alph{subsection}}

\section*{Appendix for ``Scalable Variational Causal Discovery Unconstrained by Acyclicity''}

\subsection{Theorem 2}\label{sec:Theorem-2}
\begin{theorem}
Let $\mathbb{\mathbf{A}}\in\{0,1\}^{d\times d}$ be an adjacency matrix
of a graph of $d$ nodes. Then, $\mathbf{A}$ is DAG if and only if
there exists corresponding a vector of priority scores $\mathbf{p}\in\mathbb{R}^{d}$
and a binary matrix $\mathbf{W}\in\{0,1\}^{d\times d}$ such that:
\[
\mathbf{A}=\nu\left(\mathbf{W},\mathbf{p}\right)=\mathbf{W}\circ\lim_{t\rightarrow0}\mathrm{sigmoid}\left(\frac{\mathrm{grad}(\boldsymbol{p})}{t}\right)
\]
 where $t>0$ is a strictly positive temperature and $\mathbf{p}$
contains no duplicate elements, i.e. , $p_{i}\neq p_{j}\forall i,j$.
\end{theorem}

\begin{proof}
We first show that for any DAG $\mathbf{A}$, there always exists
a pair $\left(\mathbf{W},\mathbf{p}\right)$ such that $\nu\left(\mathbf{W},\mathbf{p}\right)=\mathbf{A}$.
By leveraging Theorem 3.7 in \citep{yu2021dagswith}, we can see that
$\mathbf{p}$ implicitly define the topological order over vertices
of $\mathbf{A}$ such that:

\[
\mathrm{grad}\left(\mathbf{p}\right){}_{ij}>0\text{ when }\mathbf{A}_{ij}=1
\]

Then, 
\[
\lim_{t\rightarrow0}\text{sigmoid}\left(\frac{\mathrm{grad}\left(\mathbf{p}\right)}{t}\right)=1\text{ when }\mathbf{A}_{ij}=1
\]

Furthermore, we can choose $\mathbf{W}$ in the following way: 

\[
\mathbf{W}_{ij}=\begin{cases}
0 & \text{if }\mathbf{A}_{ij}=0\\
1 & \text{if }\mathbf{A}_{ij}=1
\end{cases}
\]

For an arbitrary topological (partial) order of the variables $\pi=\left(\pi_{1},\pi_{2},\dots,\pi_{d}\right)$,
it always defines a DAG where each edge $\left(i,j\right)$ corresponding
to $i\prec_{\pi}j$. To prove that the mapping $\nu\left(\mathbf{W},\mathbf{p}\right)$
always emits a DAG, let define a vector $\mathbf{p}\in\mathbb{R}^{d}$
such that $\mathbf{p}\left[\pi\left[i\right]\right]=i$. We have:
\begin{align*}
i\prec_{\pi}j & \Rightarrow\pi_{j}>\pi_{i}\\
 & \Rightarrow\mathbf{p}_{j}>\mathbf{p}_{i}\\
 & \Rightarrow i\prec_{\mathbf{p}}j
\end{align*}

Therefore, $\lim_{t\rightarrow0}\text{sigmoid}\left(\frac{\mathrm{grad}\left(\mathbf{p}\right)}{t}\right)$
outputs an acyclic binary adjacency matrix. Then, taking the element-wise
multiplication with any $\mathbf{W}$ gives us a sub-graph of a DAG,
which is also a DAG. 
\end{proof}

\subsection{Additional results}\label{sec:Addtional-results}

\textbf{Sampling time.} Figure \ref{fig:DAG-sampling-time} shows a the superlative running time of our proposed DAG model compared to Gumbel-Sinkhorn and Gumbel-Top-k for thousands of nodes. 

\begin{figure}[h]
\centering
\includegraphics[width=0.8\textwidth]{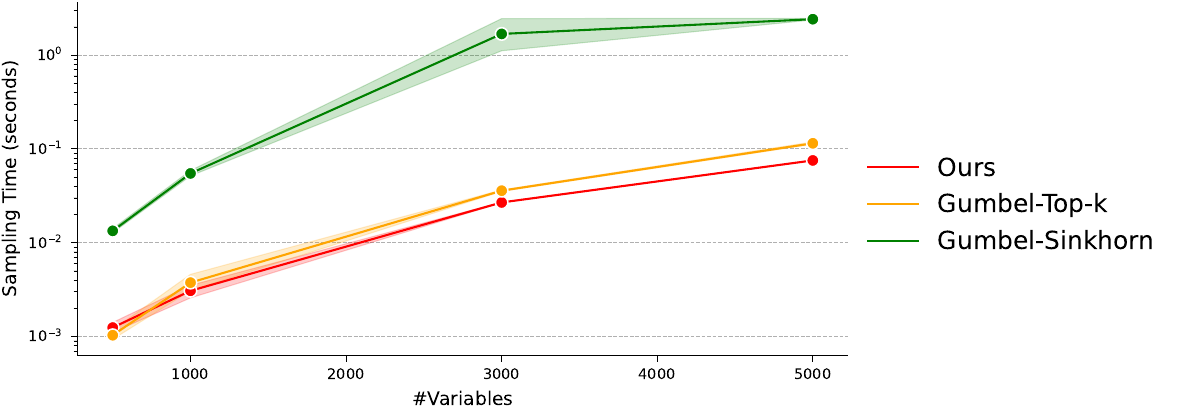}\caption{DAG sampling time in seconds of our proposed approach and two well-known DAG probabilistic models: Gumbel-Sinkhorn and Gumbel-Top-k for thousands of nodes.\newline}
\label{fig:DAG-sampling-time}
\end{figure}

\end{document}